\documentclass[11pt]{article}

\usepackage{fullpage}
\usepackage{amssymb, amsmath, amsthm}

\usepackage{bm}
\usepackage{enumerate}
\usepackage{algorithm}
\usepackage{graphicx}
\usepackage{booktabs}

\theoremstyle{plain}
\newtheorem{theo}{Theorem}

\newtheorem{lemm}{Lemma}
\newtheorem{coro}{Corollary}

\def\a{\bm{a}}

\def\c{\bm{c}}

\def\e{\bm{e}}
\def\f{\bm{f}}
\def\g{\bm{g}}

\def\p{\bm{p}}
\def\q{\bm{q}}
\def\r{\bm{r}}

\def\u{\bm{u}}
\def\v{\bm{v}}
\def\w{\bm{w}}
\def\x{\bm{x}}
\def\y{\bm{y}}
\def\z{\bm{z}}

\def\A{\bm{A}}
\def\B{\bm{B}}
\def\C{\bm{C}}
\def\D{\bm{D}}

\def\F{\bm{F}}
\def\G{\bm{G}}

\def\I{\bm{I}}

\def\L{\bm{L}}

\def\P{\bm{P}}
\def\Q{\bm{Q}}
\def\R{\bm{R}}

\def\U{\bm{U}}

\def\W{\bm{W}}

\def\Y{\bm{Y}}

\def\LC{\mathcal{L}}
\def\Sigmab{\bm{\Sigma}}

\def\trans{\top}

\def\mmin{\mathrm{min}}
\def\mmax{\mathrm{max}}

\def\Real{\mathbb{R}}

\def\SmMap{F_{\mathrm{SM}}}
\def\MsmMap{F_{\mathrm{MSM}}}
\def\NjwMap{F_{\mathrm{NJW}}}

\def\SmCenter{C_{\mathrm{SM}}}
\def\NjwCenter{C_{\mathrm{NJW}}}

\def\COST{\mathsf{COST}}

\def\OPT{\mathsf{OPT}}

\def\Prob{\mathsf{P}}

\newcommand{\by}[2]{\ensuremath{#1 \times #2}}

\title{Improved  Analysis of Spectral Algorithm for Clustering}

\author{Tomohiko~Mizutani
\thanks{Department of Mathematical and Systems Engineering,
Shizuoka University,
3-5-1 Johoku, Naka-ku, Hamamatsu City, 432-8561, Japan.
{\tt mizutani.t@shizuoka.ac.jp}}}

\date{\today}

\begin{document}

\maketitle

\begin{abstract}
Spectral algorithms are graph partitioning algorithms
that partition a node set of a graph into groups by using a spectral embedding map.
Clustering techniques based on the algorithms are referred to as spectral clustering 
and are widely used in data analysis.
To gain a better understanding of why spectral clustering is successful,
Peng et al.\ (2015) and Kolev and Mehlhorn (2016) studied the behavior of a certain type of spectral algorithm 
for a class of graphs, called well-clustered graphs.
Specifically, they put an assumption on graphs and showed the performance guarantee of the spectral algorithm under it.
The algorithm they studied used the spectral embedding map developed by Shi and Malik (2000).
In this paper, we improve on their results, 
giving a better performance guarantee under a weaker assumption.
We also evaluate the performance of the spectral algorithm with the spectral embedding map developed by Ng et al.\ (2001).

\end{abstract}

\section{Introduction}

Spectral algorithms are algorithms for partitioning a node set of a graph into groups.
They work well at clustering data and are often used for that purpose in data analysis.
The clustering techniques based on them are referred to as spectral clustering. 
Here, we will consider a graph representing relationships between items;
i.e., each node corresponds to an item and edges only connect pairs of similar nodes.
The clustering task is to partition the node set into groups such that 
nodes in the same group are well connected, while nodes in different groups are poorly connected.
Spectral algorithms tend to produce such node groups in practice.

Spectral algorithms work as follows.
The input is a graph and a positive integer $k$.
Two steps are performed on the input:
(1) a spectral embedding map $F$ is constructed that 
maps the nodes of a graph to points in a $k$-dimensional real space; 
then, a set $X$ of points $F(v)$ is formed for nodes $v$;
(2) a $k$-way partition of the node set is found 
by applying a classical clustering algorithm to $X$ (the output is the node partition).
The algorithms usually employ either of
two spectral embedding maps, i.e., the one of Shi and Malik \cite{Shi00} 
or the one of Ng et al.\ \cite{Ng02}, in the first step and $k$-means clustering in the second step.
Interested readers should see the tutorial \cite{Lux07} on spectral clustering.

Peng et al.\ studied  why spectral algorithms work well at clustering data
in the proceedings \cite{Pen15} of COLT 2015 (the full version of their paper is available at \cite{Pen17}).
The research revealed the performance of a certain type of spectral algorithm for a class of graphs, called well-clustered graphs.
Kolev and Mehlhorn improved on the results of Peng et al.,  as reported in the proceedings \cite{Kol16} of ESA 2016
(the full version is available at \cite{Kol18}).
Both studies dealt with a case where the algorithm uses the spectral embedding map of Shi and Malik.
Here, let $G$ denote a graph. They quantified the quality of a $k$-way node partition of $G$ 
by using the conductance of $G$, denoted by $\phi_k(G)$.
The use of conductance of a graph fits the aim of clustering task.
Let us say that a node partition is optimal if it achieves $\phi_k(G)$.
They measured the approximation accuracy of the output of the spectral algorithm
by how far it is from the optimal node partition.
They gave a bound on the approximation accuracy under the assumption
that there is a large gap between $\phi_k(G)$ and $\lambda_{k+1}$.
Here, $\lambda_{k+1}$ denotes the $(k+1)$th smallest eigenvalue of the normalized Laplacian of $G$.
As will see in Section \ref{Subsec: Well-clustered graphs}, 
if the assumption holds,  $G$ is called a well-clustered graph.

The contributions of this paper are twofold.
First, we  improve on the results of Peng et al.\ and Kolev and Mehlhorn (part (a) of Theorem \ref{Theo: Main}).
Our results reduce their bounds on the approximation accuracy by a factor of $k$ 
and simultaneously reduce the gap one needs to assume by a factor of $k^2$.
Second, we study a case where the spectral algorithm uses the spectral embedding map of Ng et al. (part (b) of Theorem \ref{Theo: Main}).
To the best of our knowledge, our analysis is the first one on this case.

The rest of this paper is organized as follows.
Section \ref{Sec: Preliminaries} explains the symbols and terminology used in the subsequent discussion.
Section \ref{Sec: Overview} overviews spectral algorithms and explains the concept of well-clustered graphs. 
Section \ref{Sec: Analysis of spectral algorithm} presents our results, which are summarized in Theorem \ref{Theo: Main}.
The proof of the theorem uses several theorems.
Sections \ref{Sec: Structure theorem}, \ref{Sec: Derivation of bounds} and \ref{Sec: Extended theorem} 
are devoted to establishing them.

\section{Preliminaries} \label{Sec: Preliminaries}
We denote by $G=(V,E)$ a graph with a node set $V$ and an edge set $E$.
Throughout this paper, we will deal with undirected and unweighted graphs with $n$ nodes denoted by $1, \ldots, n$.
If $k$ subsets $S_1, \ldots, S_k$ of $V$ satisfy $S_i \cap S_j = \emptyset$ for any $i \neq j$ and $\cup_{i=1}^{k} S_i = V$, 
we call  $\{S_1, \ldots, S_k\}$ a {\it $k$-way partition} of $G$ or simply a {\it partition} of $G$.
The degree of node $v$ is the number of edges at  node $v$. 
We use $d_v$ to denote the degree of $v$.
The degree matrix $\D$ of $G$ is an \by{n}{n} diagonal matrix whose $(v, v)$th entry is $d_v$.
The adjacency matrix $\W$ of $G$ is an \by{n}{n} symmetric matrix whose 
$(u,v)$th  entry  and $(v,u)$th entry  are $1$ 
if $\{u,v\} \in E$ and $0$ otherwise.
The {\it Laplacian matrix} $\L$ of $G$ is defined to be $\L = \D-\W$, 
and the {\it normalized Laplacian matrix} $\LC$ is defined to be $\LC = \D^{-1/2} \L \D^{-1/2}$ 
that is equivalent to $\I - \D^{-1/2} \L \D^{-1/2}$. 
We use $\I$ to denote the identity matrix.

Spectral algorithms are built on the eigenvalues and eigenvectors of the Laplacian or the normalized one of $G$.
In this paper, we work with those of the normalized Laplacian $\LC$.
It is known that all eigenvalues of $\LC$ are nonnegative; the smallest one is $0$, and the largest one is less than $2$.
The details can be found in  \cite{Chu97, Lux07}.
We use $\lambda_i$ to denote the $i$th smallest eigenvalue of $\LC$, and 
use $\f_i$ to denote the associated eigenvector.
As mentioned above, these $\lambda_1, \ldots, \lambda_n$ satisfy the relation $0 = \lambda_1 \le \cdots \le \lambda_n \le 2$.
Since $\LC$ is symmetric, we can choose $\f_1, \ldots, \f_n$ to be orthonormal bases in $\Real^n$.
Throughout this paper, we always choose such $\f_1, \ldots, \f_n$.

The symbol $\| \cdot \|_2$ for a vector or a matrix denotes the $2$-norm, 
and the symbol $\| \cdot \|_F$ denotes the Frobenius norm for a matrix.
Let $X$ be a finite set.
For the subsets $S$ and $T$, 
the symbol $S \triangle T$ denotes the symmetric difference of $S$ and $T$, 
i.e., $S \triangle T = (S \setminus T) \cup (T \setminus S)$.
If $k$ subsets $S_1, \ldots, S_k$ of $X$ satisfy $S_i \cap S_j = \emptyset$ for any $i \neq j$ and $\cup_{i=1}^{k} S_i = X$, 
we call $\{S_1, \ldots, S_k\}$ a {\it $k$-way partition} of $X$ or simply a {\it partition} of $X$.

\section{Overview of Spectral Algorithms and Well-Clustered Graphs} \label{Sec: Overview}

\subsection{Spectral Algorithm} \label{Subsec: Spectral algorithms}
Here, we  overview the spectral algorithm proposed by Shi and Malik \cite{Shi00} and Ng et al.\ \cite{Ng02}.
The input is a graph $G = (V,E)$ and a positive integer $k$, and the output is a $k$-way partition $\{A_1, \ldots, A_k\}$ of $G$.
The algorithm consists of two steps, as follows.

\begin{enumerate}[1.]
 \item Construct a map $F : V \rightarrow \Real^k$, called a {\it spectral embedding map},
       by using the eigenvectors of a normalized Laplacian of $G$.
       Apply $F$ to the nodes $1, \ldots, n$ of $G$ and form a set $X$ of  points $F(1), \ldots, F(n) \in \Real^k$.

 \item Find a $k$-way partition $\{X_1, \ldots, X_k\}$ of $X$ by solving the $k$-means clustering problem on $X$.
       Return $\{A_1, \ldots, A_k\}$ by letting $A_i = \{v : F(v) \in X_i \}$ for $i=1, \ldots, k$.
\end{enumerate}

The spectral embedding maps developed by Shi and Malik and Ng et al.\ are constructed as follows.
Let $\P = [\f_1, \ldots, \f_k]^\trans \in \Real^{k \times n}$ 
for the bottom $k$ eigenvectors $\f_1, \ldots, \f_k$ of the normalized Laplacian of $G$.
Let $\p_i$ be the $i$th column of $\P$.
The spectral embedding map of Shi and Malik, denoted by $\SmMap$, is given as 
\begin{equation*}
 \SmMap(v) = \frac{1}{\sqrt{d_v}} \p_v
\end{equation*}
where $d_v$ is the degree of node $v$.
That of Ng et al., denoted by $\NjwMap$, is given as
\begin{equation*}
 \NjwMap(v) = \frac{1}{\|\p_v\|_2 } \p_v.
\end{equation*}
We apply a spectral embedding map $F$ to the nodes $1, \ldots, n$ of $G$ 
and then form the set $X$ of points $F(1), \ldots, F(n) \in \Real^k$.
For any partition $\{X_1, \ldots, X_k\}$ of $X$ and any vectors $\z_1, \ldots, \z_k \in \Real^k$,
we define the clustering cost function by
\begin{align*}
 C(X_1, \ldots, X_k, \z_1, \ldots, \z_k) = \sum_{i=1}^{k} \sum_{\x \in X_i} \| \x - \z_i \|_2^2,
\end{align*}
and solve the $k$-means clustering problem $\Prob(X)$ as follows:
Given a set $X$ of points in $\Real^k$, find a $k$-way partition $\{X_1, \ldots, X_k\}$ of $X$
and vectors $\z_1, \ldots, \z_k \in \Real^k$ to minimize the clustering cost function $C$.
We call the smallest value of $C$ the {\it optimal clustering cost} of $\Prob(X)$ 
and use the symbol $\OPT$ to denote it. 
Solving $k$-means clustering problems is intractable; it was shown to be NP-hard in \cite{Alo09}.
Because of that, the spectral algorithm usually employs Lloyd's algorithm \cite{Llo82} for solving $\Prob(X)$.
In practice, Lloyd's algorithm works well on $k$-means clustering problems, 
although it does not necessarily provide the optimal clustering cost.

Next, we explain the spectral algorithm studied by Peng et al.\ \cite{Pen17} and Kolev and Mehlhorn \cite{Kol16}.
The algorithm expands each element $F(v)$ of $X$ by making $d_v$ copies of $F(v)$ and forms a set,
\begin{align*}
 Y = \{\underbrace{F(1), \ldots, F(1)}_{d_1}, \ldots, \underbrace{F(n), \ldots, F(n)}_{d_n}\}.
\end{align*}
In this paper, we say that $Y$ is the {\it expansion} of $X$.
The algorithm then solves $\Prob(Y)$ by using a $k$-means clustering algorithm with an approximation factor of $\alpha$.
Note that $\alpha$ is a real number greater than or equal to one.
Let $\{Y_1, \ldots, Y_k \}$ be the partition of $Y$ that is output by the $k$-means clustering algorithm.
From the partition $\{Y_1, \ldots, Y_k\}$, 
we can obtain a $k$-way partition $\{X_1, \ldots, X_k\}$ of $X$ by keeping one of $d_v$ copies of $F(v)$ while discarding the others.
In the analysis of Peng et al.\ \cite{Pen17} and Kolev and Mehlhorn \cite{Kol16},
it is necessary to assume  that 
the $k$-means clustering algorithm outputs a partition $\{Y_1, \ldots, Y_k\}$ of $Y$ satisfying the following condition.
\begin{description}
 \item[{\normalfont (A)}] For every $v \in V$, all $d_v$ copies of $F(v)$ are contained in one of $Y_1, \ldots, Y_k$.
\end{description}
Let a partition $\{Y_1, \ldots, Y_k\}$ of $Y$ satisfy assumption (A) and $\w_1, \ldots, \w_k$ be any vectors in $\Real^k$.
Using some partition $\{V_1, \ldots, V_k\}$ of a graph $G$, 
we can write the clustering cost $C(Y_1, \ldots, Y_k, \w_1, \ldots, \w_k)$ as 
\begin{align} \label{Exp: expression of clustering cost} 
C(Y_1, \ldots, Y_k, \w_1, \ldots, \w_k)  
 = \sum_{i=1}^{k} \sum_{v \in V_i} d_v \|F(v) - \w_i \|_2^2.
\end{align}

\subsection{Well-Clustered Graphs} \label{Subsec: Well-clustered graphs}

Here, we explain the concept of well-clustered graphs, as introduced by  Peng et al.\  \cite{Pen17}.
Let $G = (V, E)$ be a graph.
For a subset $S$ of $V$, 
the symbol $E(S, V \setminus S)$ denotes the set of all edges crossing $S$ and its complement $V \setminus S$, i.e.,
$E(S, V \setminus S) = \{\{u,v\} \in E : u \in S, \ v \in  V \setminus S\}$. 
The symbol $\mu(S)$ denotes the {\it volume} of $S$ that is given by the sum of degrees of all nodes in $S$, i.e.,
 $\mu(S) = \sum_{v \in S} d_v$. 
 The {\it conductance of a node subset $S$}, denoted by $\phi(S)$, is defined as 
 \begin{align*}
  \phi(S) = \frac{|E(S, V \setminus S)|}{\mu(S)}.
 \end{align*}
 The {\it $k$-way conductance of the graph $G$}, denoted by $\phi_k(G)$, is defined to be
 \begin{align*}
  \phi_k(G)= \min_{ \{S_1, \ldots, S_k\} } \max \{\phi(S_1), \ldots, \phi(S_k) \}.
\end{align*}
 Here, the minimum is taken over all candidates of $k$-way partitions of $G$.
 We say that a partition $\{S_1, \ldots, S_k\}$ of $G$ is  {\it $\phi_k(G)$-optimal}
 if it satisfies $\phi_k(G) = \max \{\phi(S_1), \ldots, \phi(S_k) \}$.
 Let $\Upsilon$ denote the ratio of $\lambda_{k+1}$ divided by $\phi_k(G)$, 
 \begin{align*}
  \Upsilon = \frac{\lambda_{k+1}}{\phi_k(G)}
 \end{align*}
 where $\lambda_{k+1}$ is the $(k+1)$th smallest eigenvalue of the normalized Laplacian of $G$.

 Let us put an assumption on $G$ that $\Upsilon$ is large.
 On the basis of the results of previous studies, we shall look at graphs satisfying this assumption.
 Lee et al.\ \cite{Lee12} developed a higher-order Cheeger inequality.
 It implies that,  if $\Upsilon$ is large, so is  $\lambda_{k+1} / \lambda_k$.
 Gharan and Trevisan \cite{Gha14} examined  graphs satisfying the assumption 
 that $\lambda_{k+1} / \lambda_k$ is large.
 For a subset $S$ of the node set $V$,  they introduced the inside/outside conductance of $S$.
 The outside conductance of $S$ is defined to be $\phi(S)$, 
 while the inside conductance of $S$ is defined to be $\phi_2(G[S])$, 
 that  is the two-way conductance of a subgraph $G[S]$ induced by $S$.
 They showed in Corollary 1.1 of \cite{Gha14} that,  if $\lambda_{k+1} / \lambda_k$ is large,
 there is a partition $\{S_1, \ldots, S_k\}$ of $G$ such that
 the inside conductance of each $S_1, \ldots, S_k$ is high and 
 the outside conductance of each $S_1, \ldots, S_k$ is low;
 that is to say, the nodes in $S_i$ are well connected to each other, 
 while those in $S_i$ are poorly connected to the nodes in $S_j$ with $i \neq j$.
 Accordingly, graphs are called {\it well-clustered} if they satisfy the assumption that $\Upsilon$ is large.
 The assumption is called the {\it gap assumption}.

 In connection with $\phi_k(G)$, 
 Kolev and Mehlhorn \cite{Kol16} introduced the {\it minimal average conductance}, denoted by $\bar{\phi}_k(G)$.
 Let $U$ be a set containing every $k$-way partition that is $\phi_k(G)$-optimal.  
 That is defined as
 \begin{align*}
  \bar{\phi}_k(G) = \min_{ \{S_1, \ldots, S_k\} \in U} \frac{1}{k} \left( \phi(S_1) + \cdots + \phi(S_k) \right).
 \end{align*}
 We say that a  partition $\{S_1, \ldots, S_k\}$ of $G$ is {\it $\bar{\phi}_k(G)$-optimal}
 if it satisfies $\bar{\phi}_k(G) = \frac{1}{k} (\phi(S_1) +  \cdots \phi(S_k))$.
 Let $\Psi$ denote the ratio of $\lambda_{k+1}$ divided by $\bar{\phi}_k(G)$, 
 \begin{align*}
  \Psi = \frac{\lambda_{k+1}}{\bar{\phi}_k(G)}.
 \end{align*}
 By the definition of $\bar{\phi}_k(G)$, we have the relation $\bar{\phi}_k(G) \le \phi_k(G)$.
 This implies the relation $\Psi \ge \Upsilon$.

\section{Analysis of Spectral Algorithm} \label{Sec: Analysis of spectral algorithm}

We state our results below.

\begin{theo}\label{Theo: Main}
 We are given a graph $G = (V, E)$ and a positive integer $k$.
 Let a partition $\{S_1, \ldots, S_k\}$ of $G$ be  $\bar{\phi}_k(G)$-optimal.
 Assume that a $k$-means clustering algorithm has an approximation ratio of $\alpha$
 and satisfies assumption (A).

 \begin{enumerate}[{\normalfont (a)}]
  \item  (Case of $F = \SmMap$) \
	 Let $G$ satisfy $\Psi \ge 300 k \alpha $.
	 Let $\{A_1, \ldots, A_k\}$ be a  partition of $G$ returned by the $k$-means clustering algorithm
	 for problem $\Prob(Y)$, where $Y$ is the expansion of $X = \{\SmMap(v) : v \in V\}$.
	 After a suitable renumbering of $A_1, \ldots, A_k$, the following holds for $i=1, \ldots, k$.
	 \begin{equation*}
	  \mu(A_i \triangle S_i) \le \frac{200 k \alpha }{\Psi} \mu(S_i)
	   \quad \mbox{and} \quad
           \phi(A_i) \le  \biggl( 1+ \frac{300 k \alpha }{\Psi} \biggr) \phi(S_i) + \frac{300 k \alpha }{\Psi}.
	 \end{equation*}

   \item  (Case of $F = \NjwMap$) \
	  Define $\beta = \max \{ \mu(S_1), \ldots, \mu(S_k) \} / \min \{ \mu(S_1), \ldots, \mu(S_k) \}$.
	  Let $G$ satisfy $\Psi \ge 600 k \alpha  \beta$.
	  Let $\{A_1, \ldots, A_k\}$ be a partition of $G$ returned by the $k$-means clustering algorithm
	  for problem $\Prob(Y)$, where $Y$ is the expansion of $X = \{\NjwMap(v) : v \in V\}$.
	  After a suitable renumbering of $A_1, \ldots, A_k$,
	  the following holds for $i=1, \ldots, k$.
	 \begin{equation*}
	  \mu(A_i \triangle S_i) \le \frac{300 k \alpha  \beta}{\Psi} \mu(S_i)
	   \quad \mbox{and} \quad
           \phi(A_i) \le  \biggl( 1+ \frac{600 k \alpha  \beta}{\Psi} \biggr) \phi(S_i) + \frac{600 k \alpha  \beta}{\Psi}.
	 \end{equation*}	
 \end{enumerate}
\end{theo}
The proof is given in the last part of this section.
Peng et al.\  \cite{Pen17} and Kolev and Mehlhorn \cite{Kol16}
gave a performance guarantee of the spectral algorithm in a case of $F = \SmMap$.
Table \ref{Tab: Results of performance analysis} compares
our result (part (a) of Theorem~\ref{Theo: Main}) with theirs.
We can see from the table that the result of Kolev and Mehlhorn is an improvement over that of Peng et al.;
i.e., the gap assumption is weakened due to $\Psi \ge \Upsilon$ 
and the bounds on the approximation accuracy are reduced by a factor of $k$.
Moreover, our result is an improvement over that of Kolev and Mehlhorn;
i.e., the gap is reduced by a factor of $k^2$ and
the bounds on the approximation accuracy are reduced by a factor of $k$.
To the best of our knowledge, no one has studied 
the performance of the spectral algorithm  under the gap assumption
for the case of $F = \NjwMap$.
Our result (part (b) of Theorem \ref{Theo: Main}) is the first.

\begin{table}[h]
 \centering
 \caption{Comparison of our result with the previous results in the case of $F = \SmMap$.
 The result of Peng et al.\ is shown in Theorem 1.2 of \cite{Pen17}, 
 while that of Kolev and Mehlhorn is in part (a) of Theorem 1.2 of \cite{Kol16};
 ours is in part (a) of Theorem \ref{Theo: Main}.}
 \label{Tab: Results of performance analysis}
 \begin{tabular}{lcrcl}
  \toprule
             & Gap Assumption           &  \multicolumn{3}{c}{Approximation Accuracy} \\
  \midrule
  Peng et al. 
  & $\Upsilon = \Omega(k^3)$ &  $\displaystyle \mu(A_i \triangle S_i$) & $=$ & $\displaystyle O\biggl( \frac{k^3 \alpha}{\Upsilon} \biggr) \cdot \mu(S_i)$ \\
  &                          &  $\displaystyle \phi(A_i)$ & $=$ & $\displaystyle 1.1 \phi(S_i) + O\biggl( \frac{k^3 \alpha}{\Upsilon} \biggr)$ \\
  \midrule

  Kolev and Mehlhorn 
  & $\Psi = \Omega(k^3)$ &  $\displaystyle \mu(A_i \triangle S_i)$ & $=$ & $\displaystyle O\biggl( \frac{k^2 \alpha}{\Psi} \biggr) \cdot \mu(S_i)$ \\
  &                      &  $\displaystyle \phi(A_i)$ & $=$ & $\displaystyle 1.1 \phi(S_i) + O\biggl( \frac{k^2 \alpha}{\Psi} \biggr)$ \\

    \midrule
  Our result
  & $\Psi = \Omega(k )$  &  $\displaystyle \mu(A_i \triangle S_i)$ & $=$ & $\displaystyle O\biggl( \frac{k \alpha }{\Psi} \biggr) \cdot \mu(S_i)$ \\
  &                      &  $\displaystyle \phi(A_i)$ & $=$ & $\displaystyle 1.1 \phi(S_i) + O\biggl( \frac{k \alpha }{\Psi} \biggr)$ \\

  \bottomrule
 \end{tabular}
 \end{table}

Our proof of Theorem \ref{Theo: Main} follows the strategy developed by Peng et al.\ \cite{Pen17}.
We explain this strategy in detail in Section \ref{Subsec: Proof strategy}.
For a graph $G = (V,E)$ and a spectral embedding map $F : V \rightarrow \Real^k$,
let $Y$ be the expansion of the set $X = \{ F(v) : v \in V\}$.
We employ a $k$-means clustering algorithm such that it has an approximation ratio of $\alpha$ and 
satisfies assumption (A).
Let the partition $\{Y_1, \ldots, Y_k\}$ of $Y$ and vectors $\z_1, \ldots, \z_k \in \Real^k$
be the output of the $k$-means clustering algorithm applied to problem $\Prob(Y)$.
The symbol $\COST$ denotes the value of the clustering cost function $C$ at the output, i.e.,
\begin{align*}
 \COST = C(Y_1, \ldots, Y_k, \z_1, \ldots, \z_k).
\end{align*}
Theorems \ref{Theo: Upper bound} and \ref{Theo: Lower bound}
give the lower and upper bounds on $\COST$ for when
the spectral embedding map $F$ is chosen as $F = \SmMap$ or $F = \NjwMap$.
The symbols $\mu_{\mmax}$ and $\mu_{\mmin}$ in the theorems are defined as follows.
Let $\{S_1, \ldots, S_k\}$ be a $\bar{\phi}_k(G)$-optimal partition of $G$.
We write $\mu_{\mmax} = \max \{\mu(S_1), \ldots, \mu(S_k) \}$ and $\mu_{\mmin} = \min \{\mu(S_1), \ldots, \mu(S_k) \}$.

\begin{theo} \label{Theo: Upper bound}
 Let a graph $G$ satisfy $\Psi \ge k$.
 Then, we can bound $\COST$ from above as follows.
 \begin{enumerate}[{\normalfont (a)}]
  \item $\displaystyle  \COST \le \frac{4 k \alpha }{\Psi}$ in a case of $F = \SmMap$.
  \item $\displaystyle  \COST \le \frac{16 k \mu_{\mmax} \alpha}{\Psi}$ in a case of $F = \NjwMap$.
 \end{enumerate}
\end{theo}
The proof is given in Section \ref{Sec: Derivation of bounds}.
For Theorem \ref{Theo: Lower bound} below, we still let $\{S_1, \ldots, S_k\}$ be  $\bar{\phi}_k(G)$-optimal 
and form $\{A_1, \ldots, A_k\}$ where $A_i = \{v : F(v) \in Y_i \}$
for a partition $\{Y_1, \ldots, Y_k\}$ of $Y$ output by the $k$-means clustering algorithm.
If, for any choice of $S_1, \ldots, S_k$, there is some $A_\ell$ having a large difference from it,
then $\COST$ becomes large.
The formal statement is as follows.
\begin{theo}\label{Theo: Lower bound}
 Let a graph $G$ satisfy $\Psi \ge k$.
 Let $\{A_1, \ldots, A_k\}$ and $\{S_1, \ldots, S_k\}$ be defined as above.
 We assume that, for every permutation $\pi : \{1, \ldots, k\} \rightarrow \{1, \ldots, k\}$,
 there is an index $\ell \in \{1, \ldots, k\}$  such that
 \begin{equation*}
  \mu(A_\ell \triangle S_{\pi(\ell)}) \ge 2 \epsilon  \cdot \mu(S_{\pi(\ell)}).
 \end{equation*}
 holds for a real number $\epsilon$ satisfying $ 0 \le \epsilon \le 1/2$.
 Then,  we can bound $\COST$ from below as follows.
 \begin{enumerate}[{\normalfont (a)}]
  \item $\displaystyle \COST \ge \frac{\epsilon}{8} - \frac{4 k }{\Psi}$
	in a case of $F = \SmMap$.
  \item $\displaystyle \COST \ge \frac{\epsilon \mu_{\mmin}}{4} - \frac{16 k  \mu_{\mmax}}{\Psi}$
	in a case of $F = \NjwMap$.
 \end{enumerate}
\end{theo}
The proof is given in Section \ref{Sec: Derivation of bounds}.
We can now prove Theorem \ref{Theo: Main} by contradiction:
The choice of large $\epsilon$ implies the contradiction that the lower bound exceeds the upper bound.
In the proof, we use the following inequality shown in Lemma 3.2 of \cite{Kol16}.
For the subsets $P$ and $Q$ of $V$, 
\begin{align}\label{Exp: Inequality about the number of edges}
  |E(P, V \setminus P)| \le |E(Q, V \setminus Q)| + \mu(P \triangle Q).
\end{align}

\begin{proof}[(Proof of Theorem \ref{Theo: Main})]
 \fbox{Part (a)}  \
 Choose a real number $\epsilon$ such that $\epsilon = 65 k \alpha / \Psi$.
 Obviously, $\epsilon$ is positive.
 Also, since $\Psi \ge 300 k \alpha  > 260 k \alpha$, we have $\epsilon = 65 k \alpha / \Psi < 1/4$.
 Hence, $\epsilon$ satisfies $0 \le \epsilon \le 1/4$.
 Let us assume that, for every permutation $\pi : \{1, \ldots, k\} \rightarrow \{1, \ldots, k\}$,
 there is an index $\ell$  such that $\mu(A_\ell \triangle S_{\pi(\ell)}) \ge 2 \epsilon \mu(S_{\pi(\ell)})$
 for a real number $\epsilon = 65 k \alpha / \Psi$.
 The graph $G$ satisfies $\Psi \ge 300k\alpha \ge k$.
 Hence, we can  apply Theorems \ref{Theo: Upper bound} and \ref{Theo: Lower bound}.
 Theorem \ref{Theo: Lower bound} tells us that 
 \begin{equation*}
  \COST
  \ge \frac{\epsilon}{8} - \frac{4 k }{\Psi}
   =  \left( \frac{65}{32} \right) \frac{4k\alpha}{\Psi} - \frac{4k}{\Psi}
  \ge \left( \frac{65}{32} \right) \frac{4k\alpha}{\Psi} - \frac{4k\alpha}{\Psi}
   =  \left( \frac{33}{32} \right) \frac{4k\alpha}{\Psi}.
 \end{equation*}
 We therefore reach a contradiction to  $ \COST  \le 4 k \alpha / \Psi$ shown in Theorem \ref{Theo: Upper bound}.
 The assumption is false.
 That is, after a suitable renumbering of $A_1, \ldots, A_k$, we have
 \begin{equation}
  \mu(A_i \triangle S_i) < 2 \epsilon  \mu(S_i)  = \frac{130 k \alpha }{\Psi} \mu(S_i)
   \label{Exp: SymDiffVol} 
 \end{equation}
 for every $i = 1, \ldots, k$.
 This gives the first inequality of part (a).
 Next, we  derive the second inequality.
 From inequalities (\ref{Exp: Inequality about the number of edges}) and (\ref{Exp: SymDiffVol}), 
 we have
 \begin{equation*}
  |E(A_i, V \setminus A_i)| \le |E(S_i, V \setminus S_i)| + \mu(A_i \triangle S_i) 
   < |E(S_i, V \setminus S_i)| + 2 \epsilon  \mu(S_i).
 \end{equation*}
 Also, by inequality (\ref{Exp: SymDiffVol}), 
 \begin{equation*}
  \mu(A_i) \ge \mu(A_i \cap S_i) = \mu(S_i) - \mu(S_i \setminus A_i)
   \ge \mu(S_i) - \mu(A_i \triangle S_i)
   > (1 - 2 \epsilon)  \mu(S_i).
 \end{equation*}
 Accordingly, we get
 \begin{equation*}
  \phi(A_i) = \frac{|E(A_i, V \setminus A_i)|}{\mu(A_i)}
   < \frac{1}{1 - 2 \epsilon} \phi(S_i) + \frac{2 \epsilon}{1 - 2 \epsilon}
   \le (1 + 4 \epsilon) \phi(S_i) + 4 \epsilon
   \le \biggl(1 + \frac{260 k \alpha }{\Psi} \biggr) \phi(S_i)  +  \frac{260 k \alpha }{\Psi}.
 \end{equation*}
 Here, the second inequality  uses the fact that $\epsilon$ satisfies $0 \le \epsilon \le 1/4$.
 This gives the second inequality of part (a).

 \fbox{Part (b)} \ 
 Choose a real number $\epsilon$ such that $\epsilon = 129 k \alpha  \beta / \Psi$.
 In light of $\Psi \ge 600 k \alpha \beta$, 
 we see that $\epsilon$ satisfies $0 \le \epsilon \le 1 / 4$.
 The proof is the same as in part (a).
\end{proof}

\section{Structure Theorem} \label{Sec: Structure theorem}
\subsection{Proof Strategy for Theorem \ref{Theo: Main}} \label{Subsec: Proof strategy}

Our proof of Theorem \ref{Theo: Main} follows the strategy developed by Peng et al.\ \cite{Pen17} 
for examining the performance of the spectral algorithm.
The structure theorem, shown in that paper, plays a key role in it.
Here, we give an overview of this proof strategy and explain the role of the theorem.
For a partition $\{V_1, \ldots, V_k\}$ of a graph $G = (V, E)$ and vectors $\w_1, \ldots, \w_k \in \Real^k$,
define a function $D$ by
\begin{align*} 
 D(V_1, \ldots, V_k, \w_1, \ldots, \w_k) = \sum_{i=1}^{k} \sum_{v \in V_i} d_v \|F(v) - \w_i \|_2^2
\end{align*}
where $d_v$ is the degree of node $v$ and $F$ is a spectral embedding map from $V$ to $\Real^k$.

Let a partition $\{S_1, \ldots, S_k\}$ of $G$ be $\phi_k(G)$-optimal.
The structure theorem of Peng et al.\ implies that there are $\c_1, \ldots, \c_k \in \Real^k$ such that
$\|\c_i - \c_j \|_2^2$ is lower bounded by some real number $\delta_{ij} \ge 0$ and 
$D(S_1, \ldots, S_k, \c_1, \ldots, \c_k)$ is upper bounded by some real number $\omega \ge 0$, i.e.,
\begin{align} \label{Exp: Well-clustered}
 \|\c_i - \c_j \|_2^2 \ge \delta_{ij} \ \mbox{for every} \ i \neq j \quad \mbox{and} \quad  
 D(S_1, \ldots, S_k, \c_1, \ldots, \c_k) \le \omega 
\end{align}
These inequalities imply that 
every point $F(v)$ with $v \in S_i$ is contained within a ball with center $\c_i$ and radius $\sqrt{\omega}$,
and the center $\c_i$ is at least $\sqrt{\delta_{ij}}$ away from  the center $\c_j$
of a ball of radius $\sqrt{\omega}$ that contains every point $F(v)$ for $ v \in S_j$.

In light of the inequalities of (\ref{Exp: Well-clustered}),
Peng et al.\ derived lower and upper bounds on $\COST$.
The upper bound is easy to obtain. 
We have  $\COST \le \alpha \cdot \OPT$ for 
the optimal clustering cost $\OPT$ of problem $\Prob(Y)$.
This is because the $k$-means clustering algorithm applied to $\Prob(Y)$
has an approximation ratio of $\alpha$. 
We thus see that the relation $\COST \le \alpha \cdot D(V_1, \ldots, V_k, \w_1, \ldots, \w_k)$
holds for any partition $\{V_1, \ldots, V_k\}$ of $G$ and any vectors $\w_1, \ldots, \w_k \in \Real^k$,
since $\OPT \le D(V_1, \ldots, V_k, \w_1, \ldots, \w_k)$.
Accordingly, the second inequality of (\ref{Exp: Well-clustered}) gives 
\begin{align*}
 \COST \le \alpha \cdot \omega.
\end{align*}
Deriving the lower bound needs more discussion.
We will only sketch the result here.
The details are explained in Section \ref{Sec: Derivation of bounds}.
Let a partition $\{A_1, \ldots, A_k\}$ of $G$ be the output of 
the $k$-means clustering algorithm.
We assume that, 
for every permutation $\pi : \{1, \ldots, k\} \rightarrow \{1, \ldots, k\}$,
there is an index $\ell \in \{1, \ldots, k\}$  such that
$\mu(A_\ell \triangle S_{\pi(\ell)}) \ge 2 \epsilon  \cdot \mu(S_{\pi(\ell)})$ holds 
for a real number $\epsilon$ satisfying $ 0 \le \epsilon \le 1/2$.
As will see in Section \ref{Sec: Derivation of bounds}, 
this assumption implies that $A_i$ must overlap two of those $S_1, \ldots, S_k$.
Since the points in $S_{j_1}$ are  away from those in $S_{j_2}$ with $j_1 \neq j_2$,  
$\COST$ becomes large.
Indeed,  we can show that, under the assumption, 
\begin{align*}
 \COST \ge \frac{1}{8} \sum_{i \in L} \xi_i \cdot \delta_{i p} \cdot
 \min \{ \mu(S_i), \mu(S_p) \} - \omega
\end{align*}
holds, where $L$ is a subset of $\{1, \ldots, k\}$; $p$ is an element of $\{1, \ldots, k\}$; and
$\xi_i$ is a nonnegative real number satisfying $\sum_{i \in L} \xi_i \ge \epsilon$
for the $\epsilon$ of the assumption.

The above results tell us that a larger $\delta_{ij}$ and  smaller $\omega$ provide
tighter lower and upper bounds on $\COST$.
In Section \ref{Subsec: Extension},
we show an extension of the structure theorem, and then, present a larger $\delta_{ij}$ and smaller $\omega$ 
than those derived by  Peng et al.\ \cite{Pen17} and Kolev and Mehlhorn \cite{Kol16}.
Table \ref{Tab: delta and omega} compares the values of $\delta_{ij}$ and $\omega$ in a case of $F = \SmMap$.
Deriving $\delta_{ij}$ and $\omega$ needs a condition to be put on $\Upsilon$ or $\Psi$.
Table \ref{Tab: delta and omega} includes these conditions.
From the table, we see that our values of $\delta_{ij}$ and $\omega$ are better than 
those of Peng et al.\ and Kolev and Mehlhorn, and 
our condition is weaker than theirs.
This is why our analysis provides a better approximation guarantee for the spectral  algorithm
under a weaker gap assumption.

  \begin{table}[h]
   \caption{Values of $\delta_{ij}$ and $\omega$ and conditions required for deriving those values
   in the case of $F = \SmMap$.
   The result of Peng et al.\ is from Lemmas 4.1 and 4.3 of \cite{Pen17}
   and that of Kolev and Mehlhorn is from Lemmas 2.3 and 3.1 of \cite{Kol16};
   our result is from Lemma \ref{Lemm: Well-clustered} of this paper.
   }
   \label{Tab: delta and omega}
  \centering  
  \begin{tabular}{lccc}
  \toprule
   & $\delta_{ij}$
   & $\omega$
   &  Condition \\
  \midrule       
   Peng et al.\ 
   & $\frac{1}{10^3 \cdot k \cdot \min \{ \mu(S_i), \mu(S_j) \} }$
   & $\frac{1.1k^2}{\Upsilon}$
   & $\Upsilon = \Omega(k^3) $ \\
   \midrule    
   Kolev and Mehlhorn  
   & $\frac{1}{2 \min \{\mu(S_i), \mu(S_j) \}}$    
   & $\frac{2k^2}{\Psi} $
   &  $\Psi = \Omega(k^3) $   \\
   \midrule    
   Our result 
   & $\frac{1}{\min \{\mu(S_i), \mu(S_j) \}}$       
   & $\frac{4k}{\Psi}$ 
   &  $\Psi = \Omega(k)$ \\
   \bottomrule  
  \end{tabular}
 \end{table}

\subsection{Extension} \label{Subsec: Extension}

Here, we describe the structure theorem of Peng et al. and its extension.
For a partition $\{S_1, \ldots, S_k\}$ of a graph $G$,
define the {\it indicator} $\g_i \in \Real^n$ of $S_i$
to be the vector whose $v$th element is one if $v \in S_i$ and zero otherwise.
Each nonzero position of $\g_i$ corresponds to  nodes contained in $S_i$,
and hence, $\g_i$ indicates $S_i$ among $S_1, \ldots, S_k$.
The {\it normalized indicator} $\bar{\g}_i \in \Real^n$ of $S_i$ is given as 
\begin{align*}
 \bar{\g}_i = \frac{\D^{1/2} \g_i}{\|\D^{1/2} \g_i\|_2 }
\end{align*}
for the degree matrix $\D$ of  $G$.
Since $\|\D^{1/2} \g_i\|_2$ is equivalent to  $\sqrt{\mu(S_i)}$,
the $v$th element of $\bar{\g}_i$ is $\sqrt{d_v / \mu(S_i)}$ if $v \in S_i$ and zero otherwise.

Let a partition $\{S_1, \ldots, S_k\}$ of $G$ be $\phi_k(G)$-optimal.
We take the normalized indicators $\bar{\g}_1, \ldots, \bar{\g}_k \in \Real^n$ of those $S_1, \ldots, S_k$.
We also take the bottom $k$ eigenvectors $\f_1, \ldots, \f_k \in \Real^n$ of the normalized Laplacian $\LC$ of $G$.
Peng et al.\ proved the structure theorem (Theorem 3.1 of \cite{Pen17}), which states the following:
Let $\Upsilon = \Omega(k^2)$ and $1 \le i \le k $. Then, there is a linear combination $\hat{\g}_i$ of
$\bar{\g}_1, \ldots, \bar{\g}_k$ such that $ \|\f_i - \hat{\g}_i \|_2 \le \sqrt{1.1k / \Upsilon}$.
We can extend the theorem further.
\begin{theo}\label{Theo: Structure theorem}
 Let a graph $G$ satisfy $\Psi \ge  0$.
 Let a partition $\{S_1, \ldots, S_k\}$ of $G$ be $\bar{\phi}_k(G)$-optimal.
 Form $\bar{\G} = [\bar{\g}_1, \ldots, \bar{\g}_k] \in \Real^{n \times k}$
 for the normalized indicators $\bar{\g}_1, \ldots, \bar{\g}_k$ of $S_1, \ldots, S_k$.
 Form $\F = [\f_1, \ldots, \f_k] \in \Real^{n \times k}$ for the bottom $k$ eigenvectors $\f_1, \ldots, \f_k$
 of the normalized Laplacian  of  $G$.
 Then, there is some $k \times k$ orthogonal matrix $\U$ such that
 \begin{align*} 
  \|\F \U - \bar{\G} \|_F \le  k / \Psi + \sqrt{k / \Psi}. 
 \end{align*}
\end{theo}
The proof is given in Section \ref{Sec: Extended theorem}.
The relation $\sqrt{x} \ge x$ holds for $0 \le x \le 1$.
Hence, if $\Psi \ge k$, we can simplify the inequality of the theorem as 
 \begin{align} 
  \|\F \U - \bar{\G} \|_F \le  2 \sqrt{k / \Psi}.
  \label{Exp: bound on FU - barG}
 \end{align}
 From the theorem, we immediately obtain the following corollary.
 This is because, if $\{S_1, \ldots, S_k\}$ is $\bar{\phi}_k(G)$-optimal, it is $\phi_k(G)$-optimal;
 $\|\A\|_2 \le \|\A\|_F$ holds for any matrix $\A$; and  $\Psi \ge \Upsilon$ holds.
\begin{coro}
 Let a graph $G$ satisfy $\Upsilon > 0$.
 Let a partition $\{S_1, \ldots, S_k\}$ of $G$ be $\phi_k(G)$-optimal.
 Form $\bar{\G} = [\bar{\g}_1, \ldots, \bar{\g}_k] \in \Real^{n \times k}$
 for the normalized indicators $\bar{\g}_1, \ldots, \bar{\g}_k$ of $S_1, \ldots, S_k$.
 Form $\F = [\f_1, \ldots, \f_k] \in \Real^{n \times k}$ for the bottom $k$ eigenvectors $\f_1, \ldots, \f_k$
 of the normalized Laplacian  of  $G$.
 Then, there is some $k \times k$ orthogonal matrix $\U$ such that
 \begin{align*} 
  \|\F \U - \bar{\G} \|_2 \le  k / \Upsilon + \sqrt{k / \Upsilon}. 
 \end{align*}
\end{coro}
The corollary implies the followings:
Let $\Upsilon \ge k$ and $1 \le i \le k $. Then, there is a linear combination $\hat{\g}_i$ of
$\bar{\g}_1, \ldots, \bar{\g}_k$ such that $ \|\f_i - \hat{\g}_i \|_2 \le 2\sqrt{k / \Upsilon}$.
We see that the condition on $\Upsilon$ imposed by the corollary
is weaker than that by the structure theorem of  Peng et al.

We are now ready for deriving the bounds $\delta_{ij}$ and $\omega$ shown in (\ref{Exp: Well-clustered}).
Let $\P = \F^\trans \in \Real^{k \times n}$ and $\Q = \bar{\G}^\trans \in \Real^{k \times n}$.
Inequality of (\ref{Exp: bound on FU - barG}) can be rewritten as
\begin{align*}
\| \P - \U\Q \|_F \le 2 \sqrt{k / \Psi}.
\end{align*}
Let  $v$ belong to $S_i$. Then, the $v$th column $\q_v$ of $\Q$  is given as 
\begin{align*}
 \q_v = \sqrt{\frac{d_v}{\mu(S_i)}} \e_i.
\end{align*}
Here, $\e_i$ denotes the $i$th unit vector in $\Real^k$.
Accordingly, we can express $\|\P - \U\Q \|_F^2$ as 
 \begin{align} \label{Exp: Bound on P - UQ}
 \| \P - \U \Q\|_F^2 =  \sum_{i=1}^{k} \sum_{v \in S_i} \left\| \p_v  - \sqrt{\frac{d_v}{\mu(S_i)}} \u_i \right\|_2^2
 \le  \frac{4k}{\Psi}
\end{align}
using the $v$th column $\p_v$ of $\P$ and the $i$th column $\u_i$ of $\U$.
For a partition $\{S, \ldots, S_k\}$ of $G$, define two types of centroid of $S_i$ by 
\begin{align*}
 \SmCenter(S_i) = \frac{1}{\sqrt{\mu(S_i)}} \u_i \quad \mbox{and}  \quad  \NjwCenter(S_i) = \u_i
\end{align*}
for the $i$th column $\u_i$ of $\U$.
Lemma \ref{Lemm: Well-clustered}  presents the values of the bounds $\delta_{ij}$ and $\omega$ in (\ref{Exp: Well-clustered}).
The proof of Lemma \ref{Lemm: Well-clustered} uses Lemma \ref{Lemm: Bound on the distance} below.
\begin{lemm}\label{Lemm: Bound on the distance}
 The following inequality holds for  a vector $\a \in \Real^k$ and a vector $\u \in \Real^k$ with $\|\u\|_2 = 1$:
 \begin{align*}
  \left\| \frac{\a}{\|\a\|_2} - \u \right\|_2 \le 2 \| \a - \u \|_2.
 \end{align*}
\end{lemm}
\begin{proof}
 We can bound the left-side norm from above as 
\begin{align*}
 \left\| \frac{\a}{\|\a\|_2} - \u \right\|_2
 & \le  \left\| \frac{\a}{\|\a\|_2} - \a \right\|_2  + \left\| \a - \u \right\|_2. 
\end{align*}
 Since $\|\u\|_2 = 1$, 
 the Cauchy-Schwarz inequality implies $(\a^\trans \u)^2 \le \|\a\|_2^2 \|\u\|_2^2 = \|\a\|_2^2 $.
 Hence, we have $\a^\trans \u \le |\a^\trans \u| \le \|\a\|_2$.
 In light of this relation, we have
 \begin{align*}
  \left\| \frac{\a}{\|\a\|_2} - \a \right\|_2^2
   =  \| \a \|^2_2 - 2 \| \a \|_2 + 1 
   \le  \| \a \|^2_2 - 2  \a^\trans \u  + 1 
   =  \| \a - \u \|_2^2,
 \end{align*}
 from which the lemma follows.
\end{proof}
\begin{lemm}\label{Lemm: Well-clustered} 
 Let a graph $G$ satisfy $\Psi \ge k$.
 Let a partition $\{S_1, \ldots, S_k\}$ of $G$ be  $\bar{\phi}_k(G)$-optimal.
 The following statements hold.
 \begin{enumerate}[{\normalfont (a)}]
  \item Consider the case of $F = \SmMap$.
	Let $\c_i = \SmCenter(S_i)$ for $i = 1, \ldots, k$. Then, 
	\begin{itemize}
	 \item  $\displaystyle \|\c_i - \c_j \|_2^2 \ge \frac{1}{\min \{\mu(S_i), \mu(S_j)\} }$ for every $i \neq j$, 
	 \item  $\displaystyle D(S_1, \ldots, S_k, \c_1, \ldots, \c_k) \le \frac{4k}{\Psi}$.
	\end{itemize}

  \item Consider the case of $F = \NjwMap$.
	Let  $\c_i = \NjwCenter(S_i)$ for $i = 1, \ldots, k$. Then, 
	\begin{itemize}
	 \item  $\displaystyle \|\c_i - \c_j \|_2^2 = 2 $ for every $i \neq j$, 
	 \item  $\displaystyle D(S_1, \ldots, S_k, \c_1, \ldots, \c_k) \le \frac{16 k \mu_{\mmax}}{\Psi}$.
	\end{itemize}
 \end{enumerate}
\end{lemm}
\begin{proof}
 \fbox{Part (a)} \ We have
 \begin{align*}
  \|\c_i - \c_j \|_2^2
   = \left\| \frac{1}{\sqrt{\mu(S_i)}}\u_i - \frac{1}{\sqrt{\mu(S_j)}}\u_j \right\|_2^2  
   =  \frac{1}{\mu(S_i)} + \frac{1}{\mu(S_j)} 
   \ge \frac{1}{ \min \{ \mu(S_i), \mu(S_j) \}}.
 \end{align*}
 Also, in light of inequality (\ref{Exp: Bound on P - UQ}), we have
 \begin{align*}
  D(S_1, \ldots, S_k, \c_1, \ldots, \c_k) 
   = \sum_{i=1}^{k} \sum_{v \in S_i} d_v \left\| \SmMap(v) - \SmCenter(S_i) \right\|_2^2 
   = \sum_{i=1}^{k} \sum_{v \in S_i}  \left\|  \p_v - \sqrt{\frac{d_v}{\mu(S_i)}} \u_i \right\|_2^2 
   \le \frac{4k}{\Psi}.
 \end{align*}

 \fbox{Part (b)} \ We have $\|\c_i - \c_j\|_2^2 = \|\u_i - \u_2 \|_2^2 = 2$.
 Define the modification $\MsmMap$ of the map $\SmMap$ by
 \begin{align*}
  \MsmMap(v) = \sqrt{\mu(S_i)} \cdot \SmMap(v)
 \end{align*}
 for $v \in S_i$.
 Using the above modification, we can express $\NjwMap(v)$ as $\NjwMap(v) = \MsmMap(v) / \|\MsmMap(v)\|_2$.
 From Lemma \ref{Lemm: Bound on the distance} and part (a) of this lemma, we find that 
 \begin{align*}
      D(S_1, \ldots, S_k, \c_1, \ldots, \c_k) 
  & = \sum_{i=1}^{k} \sum_{v \in S_i}     d_v \left\| \NjwMap(v) - \NjwCenter(S_i) \right\|_2^2 \\
  & = \sum_{i=1}^{k} \sum_{v \in S_i}     d_v \left\| \frac{\MsmMap(v)}{\| \MsmMap(v) \|_2} - \u_i \right\|_2^2 \\
  & \le 4 \sum_{i=1}^{k} \sum_{v \in S_i} d_v \left\| \MsmMap(v) - \u_i \right\|_2^2 \\
  & = 4 \sum_{i=1}^{k} \sum_{v \in S_i} \mu(S_i)  d_v \left\| \SmMap(v) - \SmCenter(S_i) \right\|_2^2 \\
  & \le \frac{16 k \mu_{\mmax}}{\Psi}.
  \end{align*}
\end{proof}

As explained in Section \ref{Subsec: Proof strategy},
it is desirable that the lower bound on $\|\c_i - \c_j \|_2^2$ is as large as possible
and upper bound on $D(S_1, \ldots, S_k, \c_1, \ldots, \c_k)$ is as small as possible,
since they provide tight bounds on $\COST$.
We can see from Lemma \ref{Lemm: Well-clustered} that 
the case of $F = \SmMap$ is inferior to the case of $F = \NjwMap$ in terms of the lower bound on $\|\c_i - \c_j \|_2^2$, 
while the case of $F = \SmMap$ is superior to the case of $F = \NjwMap$
in terms of the upper bound on $D(S_1, \ldots, S_k, \c_1, \ldots, \c_k)$.
We can also see from Lemma \ref{Lemm: Well-clustered} how points $F(v)$ are located in $\Real^k$.
The lemma tells us that every point $F(v)$ for $v \in S_\ell$ is contained in a ball $B_\ell$ in $\Real^k$ with the following properties:
\begin{itemize}
 \item In a case of $F = \SmMap$, 
       the radius of $B_\ell$ is $2 \sqrt{\frac{k}{\Psi}}$ and the center of $B_\ell$ is at $\frac{1}{\sqrt{\mu(S_\ell)}}\u_\ell$.
       Let $\c_\ell$ denote the center of $B_\ell$.
       Then, $\c_{\ell_1}$ is orthogonal to $\c_{\ell_2}$ for any $\ell_1 \neq \ell_2$ 
       and the distance between them is at least $\frac{1}{\sqrt{\mu_{\mmax}}}$.

 \item In a case of $F = \NjwMap$, the radius of $B_\ell$ is $4 \sqrt{\frac{k \mu_{\mmax}}{\Psi}}$
       and the center  of $B_\ell$ is at $\u_\ell$.
       Let $\c_\ell$ denote the center of $B_\ell$.
       Then, $\c_{\ell_1}$ is orthogonal to $\c_{\ell_2}$ for any $\ell_1 \neq \ell_2$ 
       and the distance between them is $\sqrt{2}$.

\end{itemize}

\section{Proofs of Theorems \ref{Theo: Upper bound} and \ref{Theo: Lower bound}} \label{Sec: Derivation of bounds}

We first prove Theorem \ref{Theo: Upper bound}, which provides an upper bound on $\COST$.

\begin{proof}[(Proof of Theorem \ref{Theo: Upper bound})]
 As explained in Section \ref{Subsec: Proof strategy}, 
 the relation 
 \begin{align*}
  \COST \le \alpha \cdot \OPT \le \alpha \cdot D(V_1, \ldots, V_k, \w_1, \ldots, \w_k)
 \end{align*}
 holds for any partition $\{V_1, \ldots, V_k\}$ of a graph $G$ and any vectors $\w_1, \ldots, \w_k \in \Real^k$.
 Here, $\OPT$ is the optimal clustering cost of problem $\Prob(Y)$ and 
 $\alpha$ is the approximation factor of the $k$-means clustering algorithm applied to $\Prob(Y)$.
 Accordingly, Lemma \ref{Lemm: Well-clustered} implies inequalities (a) and (b) of this theorem.
\end{proof}

Next, we prove Theorem \ref{Theo: Lower bound}, which provides a lower bound on $\COST$.
The proof uses Lemma 4.6 of \cite{Pen17} by Peng et al., which 
was also used in Kolev and Mehlhorn's analysis,
where it was stated as Lemma 26 of \cite{Kol18} and called the volume overlap lemma.
We restate it here as Lemma \ref{Lemm: volume overlap} before proving Lemma \ref{Lemm: Bound on the cost of g}.
Let $G = (V, E)$ be a graph.
In Lemmas \ref{Lemm: volume overlap} and \ref{Lemm: Bound on the cost of g},
we let $\{S_1, \ldots, S_k\}$ be a $\bar{\phi}_k(G)$-optimal partition of $G$ and 
$\{V_1, \ldots, V_k\}$ be any partition of $G$.
The statement of Lemma \ref{Lemm: volume overlap} uses the map  $\sigma : \{1, \ldots, k\} \rightarrow \{1, \ldots, k\}$
defined by 
\begin{align*}
 \sigma(i) = \arg \max_{j \in \{1, \ldots, k\}} \frac{\mu(V_i \cap S_j)}{\mu(S_j)}.
\end{align*}
Note that the statement of the lemma is due to Kolev and Mehlhorn 
and was stated in Lemma 26 of \cite{Kol18}.
\begin{lemm}[Lemma 4.6 of \cite{Pen17}, Lemma 26 of \cite{Kol18}]\label{Lemm: volume overlap}
 We assume that, for every permutation $\pi : \{1, \ldots, k\} \rightarrow \{1, \ldots, k\}$,
 there is an index $\ell \in \{1, \ldots, k\}$ such that
 \begin{align*}
  \mu(V_\ell \triangle S_{\pi(\ell)}) \ge 2 \epsilon \cdot \mu(S_{\pi(\ell)})
 \end{align*}
 holds for a real number $\epsilon$ satisfying $ 0 \le \epsilon \le 1/2$.
 Then, at least one of the following three cases holds.
 \begin{enumerate}[{\normalfont ({Case} 1)}]
  \item  Suppose that $\sigma$ is a permutation and
	 $\mu(S_{\sigma(\ell)} \setminus V_{\ell}) \ge \epsilon \cdot \mu(S_{\sigma(\ell)})$.
	 Then, for every $i \in \{1, \ldots, k\}$, there is a real number $\epsilon_i \ge 0$ such that
	 \begin{align*}
	  & \mu(V_i \cap S_{\sigma(i)}) \ge \epsilon_i \cdot \min\{ \mu(S_{\sigma(i)}), \mu(S_{\sigma(\ell)})  \}, \\
	  & \mu(V_i \cap S_{\sigma(\ell)}) \ge \epsilon_i \cdot \min\{ \mu(S_{\sigma(i)}), \mu(S_{\sigma(\ell)})  \},
	 \end{align*}
	 and $\epsilon_1, \ldots, \epsilon_k$ satisfy $\sum_{i=1}^{k} \epsilon_i \ge \epsilon$.

  \item  Suppose that $\sigma$ is a permutation and
	 $\mu(V_{\ell} \setminus S_{\sigma(\ell)}) \ge \epsilon \cdot \mu(S_{\sigma(\ell)})$.
	 Then, for every $i \in \{1, \ldots, k\}$, there is a real number $\epsilon_i \ge 0$ such that
	 \begin{align*}
	  & \mu(V_{\ell} \cap S_{\sigma(i)}) \ge \epsilon_i \cdot \min\{ \mu(S_{\sigma(i)}), \mu(S_{\sigma(\ell)})  \}, \\
	  & \mu(V_{\ell} \cap S_{\sigma(\ell)}) \ge \epsilon_i \cdot \min\{ \mu(S_{\sigma(i)}), \mu(S_{\sigma(\ell)})  \},
	 \end{align*}
	 and $\epsilon_1, \ldots, \epsilon_k$  satisfy $\sum_{i=1}^{k} \epsilon_i \ge \epsilon$.
	 
  \item Suppose that $\sigma$ is not a permutation.
	Choose an index $m \in \{1, \ldots, k\} \setminus \{\sigma(1), \ldots, \sigma(k)\}$.
	Then, for every $i \in \{1, \ldots, k\}$, there is a real number $\epsilon_i \ge 0$ such that
	\begin{align*}
	 & \mu(V_i \cap S_{\sigma(i)}) \ge \epsilon_i \cdot \min\{ \mu(S_{\sigma(i)}), \mu(S_m)  \}, \\
	 & \mu(V_i \cap S_m) \ge \epsilon_i \cdot \min\{ \mu(S_{\sigma(i)}), \mu(S_m)  \},
	\end{align*}
	and $\epsilon_1, \ldots, \epsilon_k$  satisfy $\sum_{i=1}^{k} \epsilon_i \ge \epsilon$.
 \end{enumerate}
\end{lemm}
For a better understanding of Lemma \ref{Lemm: volume overlap}, we should note the following relation.
Let $P$ and $Q$ be the subsets of the node set $V$ of $G$ and $c$ be a nonnegative real number.
The inequality $\mu(P \triangle Q) \ge 2c $ implies 
that at least one of  $\mu(P \setminus Q) \ge c$ or $\mu(Q \setminus P) \ge c$ holds.
This is because
$\mu(P \triangle Q) = \mu((P \setminus Q) \cup (Q \setminus P) ) = \mu(P \setminus Q) + \mu(Q \setminus P)$;
hence, if $\mu(P \setminus Q) < c$ and $\mu(Q \setminus P) < c$, we find that $\mu(P \triangle Q) < 2c$,
which contradicts $\mu(P \triangle Q) \ge 2c $.

Now we prove Lemma \ref{Lemm: Bound on the cost of g},
which provides a lower bound on $D(V_1, \ldots, V_k, \w_1, \ldots, \w_k)$
for any vectors $\w_1, \ldots, \w_k \in \Real^k$.
Here, Lemma \ref{Lemm: Well-clustered} ensures that, if $\Psi \ge k$,
there are $\c_1, \ldots, \c_k \in \Real^k$
such that the relation shown in (\ref{Exp: Well-clustered}) holds, i.e.,
\begin{align*}
 \|\c_i - \c_j \|_2^2 \ge \delta_{ij} \ \mbox{for every} \ i \neq j \quad \mbox{and} \quad  
 D(S_1, \ldots, S_k, \c_1, \ldots, \c_k) \le \omega.
\end{align*}
Lemma \ref{Lemm: Bound on the cost of g} uses the above $\delta_{ij}$ and $\omega$ 
to describe the lower bound of $D(V_1, \ldots, V_k, \w_1, \ldots, \w_k)$.
The proof uses the following inequalities shown in \cite{Pen17}.
Let $\p, \q, \r \in \Real^k$. Then,
\begin{align} \label{Exp: Inequality about p and q}
 \| \p \|_2^2 \ge \frac{1}{2} \| \p - \q \|_2^2 - \| \q \|_2^2.
\end{align}
Also, 
\begin{align}  \label{Exp: Inequality about p, q and r}
 \| \p - \r \|_2 \ge \| \q - \r \|_2 \Longrightarrow  \| \p - \r \|_2 \ge \frac{1}{2} \| \p - \q \|_2.
\end{align}
\begin{lemm} \label{Lemm: Bound on the cost of g}
 Let a graph $G$ satisfy $\Psi \ge k$.  Let $\w_1, \ldots, \w_k$ be any vectors in $\Real^k$.
 Suppose that the assumption of Lemma \ref{Lemm: volume overlap} holds.
 Then, the inequality 
 \begin{align*}
  D(V_1, \ldots, V_k, \w_1, \ldots, \w_k) \ge \frac{1}{8} \sum_{i \in L} \xi_i \cdot \delta_{i p} \cdot
  \min \{ \mu(S_i), \mu(S_p) \} - \omega
 \end{align*}
 holds for a subset $L$ of  $\{1, \ldots, k\}$,
 an element $p$ of $\{1, \ldots, k\}$,
 and a real number $\xi_i \ge 0$
 satisfying $\sum_{i \in L} \xi_i \ge \epsilon$, where $\epsilon$ is as in the assumption.
\end{lemm}
The proof is essentially the same as those of Lemma 4.5 in \cite{Pen17} and Lemma 27 in \cite{Kol18}.
We have included it below to make the discussion self-contained.
\begin{proof}
 Let us evaluate a lower bound on $D(V_1, \ldots, V_k, \w_1, \ldots, \w_k)$
 for the three cases considered in Lemma \ref{Lemm: volume overlap}.
 In what follows, we let $\c_1, \ldots, \c_k \in \Real^k$ satisfy the relation shown in (\ref{Exp: Well-clustered}).

 \fbox{Case 1} \
 Define the map  $\gamma : \{1, \ldots, k\} \rightarrow \{1, \ldots, k\}$ by
 \begin{align*}
  \gamma(i) =
  \left\{
  \begin{array}{ll}
   \sigma(i)    & \mbox{if} \ \|\c_{\sigma(i)} - \w_i \|_2 \ge \|\c_{\sigma(\ell)} - \w_i  \|_2,  \\
   \sigma(\ell) & \mbox{otherwise}.
  \end{array}
  \right.
 \end{align*}
 In light of inequality (\ref{Exp: Inequality about p and q}), 
 we lower bound $D(V_1, \ldots, V_k, \w_1, \ldots, \w_k)$ as, 
 \begin{align*}
  D(V_1, \ldots, V_k, \w_1, \ldots, \w_k)
  & = \sum_{i=1}^{k} \sum_{v \in V_i} d_v \|F(v) - \w_i \|_2^2  \\
  & \ge \sum_{i=1}^{k} \sum_{v \in V_i \cap S_{\gamma(i)}} d_v \|F(v) - \w_i \|_2^2   \\
  & \ge \frac{1}{2} \sum_{i=1}^{k} \sum_{v \in V_i \cap S_{\gamma(i)}} d_v \|\c_{\gamma(i)} - \w_i \|_2^2
   - \sum_{i=1}^{k} \sum_{v \in V_i \cap S_{\gamma(i)}} d_v \|F(\v) - \c_{\gamma(i)} \|_2^2.
 \end{align*}
 The second term of the right side is
 \begin{align*}
  \sum_{i=1}^{k} \sum_{v \in V_i \cap S_{\gamma(i)}} d_v \|F(\v) - \c_{\gamma(i)} \|_2^2
  \le D(S_1, \ldots, S_k, \c_1, \ldots, \c_k) \le \omega.
 \end{align*}
 The first term of the right side is 
 \begin{align*}
  \frac{1}{2} \sum_{i=1}^{k} \sum_{v \in V_i \cap S_{\gamma(i)}} d_v \|\c_{\gamma(i)} - \w_i \|_2^2
  & = \frac{1}{2} \sum_{i=1}^{k} \mu(V_i \cap S_{\gamma(i)}) \cdot \|\c_{\gamma(i)} - \w_i \|_2^2 \\
  & \ge \frac{1}{8} \sum_{i=1}^{k} \mu(V_i \cap S_{\gamma(i)}) \cdot \delta_{\sigma(i) \sigma(\ell)} \\
  & \ge \frac{1}{8} \sum_{i=1}^{k} \epsilon_i \cdot \delta_{\sigma(i) \sigma(\ell)}
  \cdot \min \{\mu(S_{\sigma(i)}), \mu(S_{\sigma(\ell)}) \} 
 \end{align*}
 where $\sum_{i=1}^{k} \epsilon_i \ge \epsilon$.
 Here, the first inequality follows from inequality (\ref{Exp: Inequality about p, q and r}),
 and the second inequality follows form Lemma \ref{Lemm: volume overlap}. 
 Consequently, we obtain
 \begin{align*}
  D(V_1, \ldots, V_k, \w_1, \ldots, \w_k)  \ge
  \frac{1}{8} \sum_{i=1}^{k} \epsilon_i \cdot \delta_{\sigma(i) \sigma(\ell)}
  \cdot \min \{\mu(S_{\sigma(i)}), \mu(S_{\sigma(\ell)}) \} - \omega,
 \end{align*}
 which implies the inequality of this lemma.

 \fbox{Case 2} \
 We can assume that 
 \begin{align} \label{Exp: Volume inequality}
  \mu(V_{\ell} \cap S_{\sigma(\ell)}) \ge \epsilon \cdot \mu(S_{\sigma(\ell)})
 \end{align}
 holds. Indeed, if $\mu(V_{\ell} \cap S_{\sigma(\ell)}) < \epsilon \cdot \mu(S_{\sigma(\ell)})$, then,
 \begin{align*}
  \mu(S_{\sigma(\ell)} \setminus V_\ell)  = \sum_{i \in \{1, \ldots, k \} \setminus \{\ell\}} \mu(S_{\sigma(\ell)} \cap V_i)
  = \mu(S_{\sigma(\ell)}) - \mu(S_{\sigma(\ell)} \cap V_{\ell})
  > (1-\epsilon) \mu(S_{\sigma(\ell)}) \ge \epsilon \cdot \mu(S_{\sigma(\ell)}).
 \end{align*}
 This is in case 1. Here, the last inequality comes from the fact that $\epsilon$ satisfies $0 \le \epsilon \le 1/2$.

 First, consider the case in which there is an index $i \in \{1, \ldots, k\}$ such that
 the inequality $\|\c_{\sigma(i)} - \w_{\ell} \|_2 < \|\c_{\sigma(\ell)} - \w_\ell \|_2$ holds.
 We lower bound $D(V_1, \ldots, V_k, \w_1, \ldots, \w_k)$ as, 
 \begin{align*}
  D(V_1, \ldots, V_k, \w_1, \ldots, \w_k)
  & = \sum_{i=1}^{k} \sum_{v \in V_i} d_v \|F(v) - \w_i \|_2^2  \\
  & \ge \sum_{v \in V_\ell \cap S_{\sigma(\ell)}} d_v \|F(v) - \w_\ell \|_2^2   \\
  & \ge \frac{1}{2} \sum_{v \in V_\ell \cap S_{\sigma(\ell)}} d_v \|\c_{\sigma(\ell)} - \w_\ell \|_2^2
   - \sum_{v \in V_\ell \cap S_{\sigma(\ell)}} d_v \|F(\v) - \c_{\sigma(\ell)} \|_2^2.
 \end{align*}
 The second term of the right side is
 \begin{align*}
  \sum_{v \in V_\ell \cap S_{\sigma(\ell)}} d_v \|F(\v) - \c_{\sigma(\ell)} \|_2^2
  \le D(S_1, \ldots, S_k, \c_1, \ldots, \c_k) \le \omega.
 \end{align*}
 The first term of the right side is 
 \begin{align*}
  \frac{1}{2} \sum_{v \in V_\ell \cap S_{\sigma(\ell)}} d_v \|\c_{\sigma(\ell)} - \w_\ell \|_2^2
  & = \frac{1}{2} \mu(V_\ell \cap S_{\sigma(\ell)}) \cdot \|\c_{\sigma(\ell)} - \w_\ell \|_2^2 \\
  & \ge \frac{1}{8} \mu(V_\ell \cap S_{\sigma(\ell)}) \cdot \delta_{\sigma(i) \sigma(\ell)} \\
  & \ge \frac{1}{8}  \epsilon \cdot \delta_{\sigma(i) \sigma(\ell)} \cdot  \mu(S_{\sigma(\ell)}) \\
  & \ge \frac{1}{8}  \epsilon \cdot \delta_{\sigma(i) \sigma(\ell)} \cdot  \min \{\mu(S_{\sigma(i)}), \mu(S_{\sigma(\ell)}) \}. 
 \end{align*}
 The second inequality follows from that of (\ref{Exp: Volume inequality}).
 Consequently, we obtain
 \begin{align*}
  D(V_1, \ldots, V_k, \w_1, \ldots, \w_k)  \ge
  \frac{1}{8} \epsilon \cdot \delta_{\sigma(i) \sigma(\ell)}
  \cdot \min \{\mu(S_{\sigma(i)}), \mu(S_{\sigma(\ell)}) \} - \omega,
 \end{align*}
 which implies the inequality of this lemma.

 Next, consider the case in which, for every $i \in \{1, \ldots, k\}$,
 the inequality $\|\c_{\sigma(i)} - \w_{\ell} \|_2 \ge \|\c_{\sigma(\ell)} - \w_\ell \|_2$ holds.
 We lower bound $D(V_1, \ldots, V_k, \w_1, \ldots, \w_k)$ as, 
 \begin{align*}
  D(V_1, \ldots, V_k, \w_1, \ldots, \w_k)
  & =  \sum_{i=1}^{k} \sum_{v \in V_i} d_v \|F(v) - \w_i \|_2^2  \\
 & \ge \sum_{v \in V_\ell} d_v \|F(v) - \w_\ell \|_2^2   \\
  & \ge \sum_{i=1}^{k} \sum_{v \in V_\ell \cap S_{\sigma(i)}} d_v \|F(v) - \w_\ell \|_2^2   \\
  & \ge \frac{1}{2} \sum_{i=1}^{k} \sum_{v \in V_\ell \cap S_{\sigma(i)}} d_v \|\c_{\sigma(i)} - \w_\ell \|_2^2
   - \sum_{i=1}^{k} \sum_{v \in V_\ell \cap S_{\sigma(i)}} d_v \|F(\v) - \c_{\sigma(i)} \|_2^2.
 \end{align*}
 The second term of the right side is
 \begin{align*}
  \sum_{i=1}^{k} \sum_{v \in V_\ell \cap S_{\sigma(i)}} d_v \|F(\v) - \c_{\sigma(i)} \|_2^2
  \le D(S_1, \ldots, S_k, \c_1, \ldots, \c_k) \le \omega.
 \end{align*}
 The first term of the right side is 
 \begin{align*}
  \frac{1}{2} \sum_{i=1}^{k} \sum_{v \in V_\ell \cap S_{\sigma(i)}} d_v \|\c_{\sigma(i)} - \w_\ell \|_2^2
  & = \frac{1}{2} \sum_{i=1}^{k} \mu(V_\ell \cap S_{\sigma(i)}) \cdot \|\c_{\sigma(i)} - \w_\ell \|_2^2 \\
  & \ge \frac{1}{8} \sum_{i=1}^{k} \mu(V_\ell \cap S_{\sigma(i)}) \cdot \delta_{\sigma(i) \sigma(\ell)} \\
  & \ge \frac{1}{8} \sum_{i=1}^{k} \epsilon_i \cdot \delta_{\sigma(i) \sigma(\ell)} \cdot  \min \{\mu(S_{\sigma(i)}), \mu(S_{\sigma(\ell)}) \}
 \end{align*}
 where $\sum_{i=1}^{k} \epsilon_i \ge \epsilon$.
 Consequently, we obtain
 \begin{align*}
  D(V_1, \ldots, V_k, \w_1, \ldots, \w_k)  \ge
  \frac{1}{8} \sum_{i=1}^{k} \epsilon_i \cdot \delta_{\sigma(i) \sigma(\ell)}
  \cdot \min \{\mu(S_{\sigma(i)}), \mu(S_{\sigma(\ell)}) \} - \omega,
 \end{align*}
 which implies the inequality of this lemma.

 \fbox{Case 3} \  The proof is the same as in case 1
 since the statement of case 3 in Lemma \ref{Lemm: volume overlap} matches that of case 1
 after replacing $m$ with $\sigma(\ell)$.
 \end{proof}

 Accordingly, we are now in a position to prove Theorem \ref{Theo: Lower bound}.
 \begin{proof}[(Proof of Theorem \ref{Theo: Lower bound})]
  As mentioned in Section \ref{Subsec: Spectral algorithms},
  a clustering cost can be written  in (\ref{Exp: expression of clustering cost}) if assumption (A) holds.
  Thus, $\COST$ can be written as 
  \begin{align*}
   \COST = D(A_1, \ldots, A_k, \z_1, \ldots, \z_k)
  \end{align*}
  using a  partition $\{A_1, \ldots, A_k\}$ of $G$ and vectors $\z_1, \ldots, \z_k \in \Real^k$
  returned by the $k$-means clustering algorithm applied to the problem $\Prob(Y)$, since assumption (A) holds.
  Hence, we obtain inequalities (a) and (b) of this theorem
  from Lemmas \ref{Lemm: Well-clustered} and  \ref{Lemm: Bound on the cost of g}.
 \end{proof}

\section{Proof of Theorem \ref{Theo: Structure theorem}} \label{Sec: Extended theorem} 
Let $\{S_1, \ldots, S_k\}$ be a  partition of a graph $G = (V,E)$.
The conductance $\phi(S_i)$ of $S_i$ can be expressed as 
$\phi(S_i) = \bar{\g}_i^\trans \LC \bar{\g}_i$ 
using the normalized indicator $\bar{\g}_i$ of $S_i$ and the normalized Laplacian $\LC$ of $G$.
This comes from the observation that $\phi(S_i)$ can be expressed as 
\begin{align*}
 \phi(S_i)
 = \frac{\g_i^\trans \L \g_i}{\g_i^\trans \D \g_i}
 = \frac{(\D^{1/2}\g_i)^\trans (\D^{-1/2} \L \D^{-1/2}) (\D^{1/2} \g_i)}{\| \D^{1/2} \g_i \|_2^2}
 = \bar{\g}_i^\trans \LC \bar{\g}_i,
\end{align*}
since $| E(S_i, V \setminus S_i) |$ is equivalent to $\g_i^\trans \L \g_i$ and
$\mu(S_i)$ is equivalent to $\g_i^\trans \D \g_i$.
Here, $\L$ and $\D$ are the Laplacian and the degree matrix of $G$, and $\g_i$ is the indicator of $S_i$.
In discussion below, we use the expression  $\phi(S_i) = \bar{\g}_i^\trans \LC \bar{\g}_i$.

Let $\{S_1, \ldots, S_k\}$ be $\bar{\phi}_k(G)$-optimal.
Note that $\bar{\g}_1, \ldots, \bar{\g}_k \in \Real^n$ are their normalized indicators.
We choose the eigenvectors $\f_1, \ldots, \f_n \in \Real^n$ of $\LC$.
Since they serve as orthonormal bases in $\Real^n$,
we can expand $\bar{\g}_i$ as   
\begin{align*}
 \bar{\g}_i = c_{i, 1} \f_1 + \cdots + c_{i, n} \f_n
\end{align*}
for real numbers $c_{i,1}, \ldots, c_{i,n}$.
Considering this expression of $\bar{\g}_i$, we lower bound $\phi(S_i) = \bar{\g}_i^\trans \LC \bar{\g}_i$ as
\begin{align*}
 \phi(S_i)
 & = \bar{\g}_i^\trans \LC \bar{\g}_i \\
 & = \lambda_1 c_{i,1}^2 + \cdots + \lambda_n c_{i,n}^2  \\
 & \ge \lambda_1 c_{i,1}^2 + \cdots + \lambda_k c_{i,k}^2 + \lambda_{k+1} (c_{i, k+1}^2 + \cdots + c_{i, n}^2) \\
 & \ge \lambda_{k+1} (c_{i, k+1}^2 + \cdots + c_{i, n}^2).
\end{align*}
Recall that $\lambda_1, \ldots, \lambda_n$ denote the eigenvalues of $\LC$.
The second equality is obtained by considering the eigenvalue decomposition of $\LC$.
The first and second inequalities follow from the properties of the eigenvalues $\lambda_1, \ldots, \lambda_n$ of $\LC$ 
that satisfy $0 = \lambda_1 \le \cdots \le \lambda_n \le 2$.
We truncate the expression $\bar{\g}_i = c_{i, 1} \f_1 + \cdots + c_{i, n} \f_n$ at the $k$th term
to get the expression $\hat{\f}_i$,
\begin{align*}
 \hat{\f}_i = c_{i, 1} \f_1 + \cdots + c_{i, k} \f_k.
\end{align*}
Since  $\| \bar{\g}_i - \hat{\f}_i \|_2^2 = c_{i, k+1}^2 + \cdots + c_{i, n}^2$,
the lower bound on $\phi(S_i)$ shown above can be rewritten as 
$\lambda_{k+1} \| \bar{\g}_i - \hat{\f}_i \|_2^2$.
Accordingly, we obtain 
\begin{align*}
 \bar{\phi}_k(G) = \frac{1}{k} \sum_{i=1}^{k} \phi(S_i)
 \ge \frac{\lambda_{k+1}}{k} \sum_{i=1}^{k}\| \bar{\g}_i - \hat{\f}_i \|_2^2
 = \frac{\lambda_{k+1}}{k} \|\bar{\G} - \hat{\F} \|_F^2.
\end{align*}
Here, let $\bar{\G} = [\bar{\g}_1, \ldots, \bar{\g}_k] \in \Real^{n \times k}$ and
$\hat{\F} = [\hat{\f}_1, \ldots, \hat{\f}_k] \in \Real^{n \times k}$.
Note that $\hat{\f}_i$ is represented as $\hat{\f}_i = \F \c_i$, 
where $\F = [\f_1, \ldots, \f_k] \in \Real^{n \times k}$ and $\c_i = [c_{i,1}, \ldots, c_{i,k}]^\trans \in \Real^k$; 
$\hat{\F}$ is represented as $\hat{\F} = \F \C$, where $\C = [\c_1, \ldots, \c_k] \in \Real^{k \times k}$.
Since $\Psi > 0$, the above relation can be rewritten as 
\begin{align} \label{Exp: hat_F and bar_G}
 \| \hat{\F} - \bar{\G} \|_F^2 = \|\F \C - \bar{\G} \|_F^2  \le  k / \Psi.
\end{align}

We will show that $\C$ is close to being orthogonal.
Indeed, for the columns $\c_i$ and $\c_j$ of $\C$,  we have 
\begin{align} \label{Exp: Relation of hat_f and c}
 \c_i^\trans \c_j = \c_i^\trans \F^\trans \F \c_j = \hat{\f}_i^\trans \hat{\f}_j.
\end{align}
Inequality (\ref{Exp: hat_F and bar_G}) tells us that the $i$th column $\hat{\f}_i$ of $\hat{\F}$
is close to the $i$th column $\bar{\g}_i$ of $\bar{\G}$ if $\Psi$ is large.
Since $\bar{\g}_1, \ldots, \bar{\g}_k$ are orthonormal,
we can see that  $\c_i^\trans \c_j$ is close to one if $i=j$ and zero otherwise.
The following lemma justifies this observation.
A similar observation was made by Peng et al.\ \cite{Pen17} and Kolev and Mehlhorn \cite{Kol18}.

\begin{lemm}\label{Lemm: C is close to being orthogonal}
 Let $\C, \hat{\F}$ and $\bar{\G}$ be defined as above.
 Let $\Y = \hat{\F} - \bar{\G}$.
 Then,  $\Y^\trans \Y = \I - \C^\trans \C$ holds.
\end{lemm}
\begin{proof}
 For the $i$th column $\hat{\f}_i$ of $\hat{\F}$ and the $j$th column $\bar{\g}_j$ of $\bar{\G}$, 
 we have
 \begin{align*}
  \hat{\f}_i^\trans \bar{\g}_j
  = (c_{i, 1} \f_1 + \cdots + c_{i, k} \f_k)^\trans (c_{j, 1} \f_1 + \cdots + c_{j, n} \f_n)
  = c_{i,1}c_{j,1} + \cdots + c_{i,k}c_{j,k}
  = \c_i^\trans \c_j.
 \end{align*}
 Hence, by including (\ref{Exp: Relation of hat_f and c}), we obtain
 \begin{align*}
  \c_i^\trans \c_j = \hat{\f}_i^\trans \hat{\f}_j = \hat{\f}_i^\trans \bar{\g}_j.
 \end{align*}
 Let $\y_i$ be the $i$th column of $\Y$, which is given as $\y_i = \hat{\f}_i - \bar{\g}_i$.
 In light of the above relation,  the $(i,j)$th element of $\Y^\trans \Y$ is 
 \begin{align*}
  \y_i^\trans \y_j
  = (\hat{\f}_i - \bar{\g}_i)^\trans (\hat{\f}_j - \bar{\g}_j)
  = \hat{\f}_i^\trans \hat{\f}_j - \hat{\f}_i^\trans \bar{\g}_j  - \bar{\g}_i^\trans \hat{\f}_j + \bar{\g}_i^\trans \bar{\g}_j
  = - \c_i^\trans \c_j + \bar{\g}_i^\trans \bar{\g}_j.
 \end{align*} 
 Hence, $\y_i^\trans \y_j = 1 - \c_i^\trans \c_j$ if $i = j$; otherwise, $\y_i^\trans \y_j = -\c_i^\trans \c_j$.
 This means that $\Y^\trans \Y = \I - \C^\trans \C$ holds.
\end{proof}

We are now ready to prove Theorem \ref{Theo: Structure theorem}.
\begin{proof}[(Proof of Theorem \ref{Theo: Structure theorem})]
 Let $\C$ be defined as above.  
 It follows from Lemma \ref{Lemm: C is close to being orthogonal} and inequality of (\ref{Exp: hat_F and bar_G}) that
 \begin{align*}
  \|\I - \C^\trans \C\|_F = \|\Y^\trans \Y \|_F \le \|\Y\|_F^2 = \|\hat{\F} - \bar{\G} \|_F^2  \le k / \Psi.
 \end{align*}
 We consider the singular value decomposition of $\C$, 
 given as $\C = \A \Sigmab \B^\trans$. 
 Here, $\A$ and $\B$ are \by{k}{k} orthogonal matrices,
 and $\Sigmab$ is a \by{k}{k} diagonal matrix 
 whose all diagonal elements are nonnegative. 
 Let $\sigma_1, \ldots, \sigma_k$ denote  the diagonal elements of $\Sigmab$.
 This decomposition implies 
 \begin{align*}
  \|\I - \C^\trans \C \|_F 
  & = \|\I - \B \Sigmab^2 \B^\trans \|_F \\
  & = \|\B (\I - \Sigmab^2) \B^\trans \|_F \\
  & = \|\I - \Sigmab^2 \|_F \\
  & = \sqrt{(1 - \sigma_1^2)^2 + \cdots + (1 - \sigma_k^2)^2} \\
  & = \sqrt{(1 - \sigma_1)^2 (1 + \sigma_1)^2 + \cdots + (1 - \sigma_k)^2(1 + \sigma_k)^2} \\
  & \ge \sqrt{(1 - \sigma_1)^2 + \cdots + (1 - \sigma_k)^2} \\
  & = \| \I - \Sigmab \|_F.
 \end{align*}
 We thus get the inequality $\| \I - \Sigmab \|_F \le k / \Psi$.
 Let us represent $\C$ as 
 \begin{align*}
  \C = \U + \R  
 \end{align*}
 by letting $\U = \A \B^\trans \in \Real^{k \times k}$ and $\R = \A (\Sigmab - \I) \B^\trans \in \Real^{k \times k}$.
 We then observe the followings: $\U$ is orthogonal since it satisfies $\U^\trans \U = \U \U^\trans = \I$;
 and $\R$ satisfies $\|\R\|_F \le k / \Psi$ 
 since 
 \begin{align*}
  \|\R \|_F = \|\A (\Sigmab - \I) \B^\trans \|_F = \|\Sigmab - \I\|_F \le k / \Psi.
 \end{align*}
 Using these observations and inequality of (\ref{Exp: hat_F and bar_G}),
 we find that
 \begin{align*}
  \|\F \U - \bar{\G} \|_F
  & = \|\F \U - \hat{\F} + \hat{\F} - \bar{\G} \|_F \\
  & \le \|\F \U - \hat{\F} \|_F + \| \hat{\F} - \bar{\G} \|_F \\
  & = \|\F\R \|_F + \| \hat{\F} - \bar{\G} \|_F \\
  & = \|\R \|_F + \| \hat{\F} - \bar{\G} \|_F \\
  & \le k / \Psi + \sqrt{k / \Psi}.
 \end{align*}
 Consequently, we see that
 there is some orthogonal matrix $\U$ such that $  \|\F \U - \bar{\G} \|_F \le k / \Psi + \sqrt{k / \Psi}$.
\end{proof}

\section{Concluding Remarks} \label{Sec: Concluding remarks}
We examined the performance of the spectral algorithm studied by Peng et al.\ \cite{Pen17} and Kolev and Mehlhorn \cite{Kol16},
and obtained a better performance guarantee under a weaker assumption.
Although they dealt with the case of using the spectral embedding map $\SmMap$ developed by Shi and Malik \cite{Shi00},
we additionally examined the case of using the spectral embedding map $\NjwMap$ developed by Ng et al.\ \cite{Ng02}. 
We summarized our results as Theorem \ref{Theo: Main}.

We close this paper by suggesting directions for future research.
Theorem \ref{Theo: Main} implies that 
a spectral algorithm with $\SmMap$ outperforms the one with $\NjwMap$.
This result poses a question as to whether we can prove that the performance gap is reduced.
So far, various types of spectral algorithms have been developed for tackling graph partitioning problems. 
For clustering tasks,
it is standard practice to use the algorithms of Shi and Malik \cite{Shi00} and Ng et al.\ \cite{Ng02}.
The spectral algorithm we examined differs from these standard ones.
Let $F$ be a spectral embedding map and $X$ be the set of points $F(v)$ for nodes $v$ of a graph.
The standard algorithms apply a $k$-means clustering algorithm to problem $\Prob(X)$, 
while our algorithm first forms the expansion $Y$ of $X$ and then applies 
a $k$-means clustering algorithm satisfying assumption (A) to problem $\Prob(Y)$.
While we could examine the performance of the standard algorithms,
wherein we obtain a performance guarantee by following the proof of Theorem \ref{Theo: Main},
the guarantee so obtained would be worse than the one stated in the theorem.
It remains an open question as to  whether, for standard algorithms, 
we can give a performance guarantee at a similar level to Theorem \ref{Theo: Main}.

\section*{Acknowledgements}
The author would like to thank one of the referees 
for pointing out a flaw in the previous version of Theorem 4. 
It was fixed in the current version.
This research was supported by the Japan Society for the Promotion of Science
(JSPS KAKENHI Grant Numbers 15K20986).

\bibliographystyle{abbrv}
\bibliography{reference}

\begin{thebibliography}{10}

\bibitem{Alo09}
D.~Aloise, A.~Deshpande, P.~Hansen, and P.~Popat.
\newblock {NP}-hardness of {E}uclidean sum-of-squares clustering.
\newblock {\em Machine Learning}, 75(2):245--248, 2009.

\bibitem{Chu97}
F.~R.~K. Chung.
\newblock {\em Spectral Graph Theory}.
\newblock American Mathematical Society, 1997.

\bibitem{Gha14}
S.~O. Gharan and L.~Trevisan.
\newblock Partitioning into expanders.
\newblock In {\em Proceedings of the 25 annual ACM-SIAM symposium on Discrete
  algorithms (SODA)}, pages 1256--1266, 2014.

\bibitem{Kol16}
P.~Kolev and K.~Mehlhorn.
\newblock A note on spectral clustering.
\newblock In {\em 24th Annual European Symposium on Algorithms (ESA 2016)},
  volume~57, pages 57:1--57:14, 2016.

\bibitem{Kol18}
P.~Kolev and K.~Mehlhorn.
\newblock Approximate spectral clustering: Efficiency and guarantees.
\newblock arXiv:1509.09188v5, 2018.

\bibitem{Lee12}
J.~R. Lee, S.~O. Gharan, and L.~Trevisan.
\newblock Multi-way spectral partitioning and higher-order cheeger
  inequalities.
\newblock In {\em Proceedings of the 44th annual ACM symposium on Theory of
  computing (STOC)}, pages 1117--1130, 2012.

\bibitem{Llo82}
S.~P. Lloyd.
\newblock Least squares quantization in {PCM}.
\newblock {\em IEEE Transactions on Information Theory}, 28(2):129--137, 1982.

\bibitem{Ng02}
A.~Y. Ng, M.~Jordan, and Y.~Weiss.
\newblock On spectral clustering: Analysis and an algorithm.
\newblock In {\em Advances in Neural Information Processing Systems 14 (NIPS)},
  pages 849--856, 2001.

\bibitem{Pen15}
R.~Peng, H.~Sun, and L.~Zanetti.
\newblock Partitioning well-clustered graphs: Spectral clustering works!
\newblock In {\em Proceedings of the 28th Conference on Learning Theory
  (COLT)}, volume~40, pages 1423--1455, 2015.

\bibitem{Pen17}
R.~Peng, H.~Sun, and L.~Zanetti.
\newblock Partitioning well-clustered graphs: Spectral clustering works!
\newblock {\em SIAM Journal on Computing}, 46(2):710--743, 2017.

\bibitem{Shi00}
J.~Shi and J.~Malik.
\newblock Normalized cuts and image segmentation.
\newblock {\em IEEE Transactions on Pattern Analysis and Machine Intelligence},
  22(8):888--905, 2000.

\bibitem{Lux07}
U.~von Luxburg.
\newblock A tutorial on spectral clustering.
\newblock {\em Statistics and Computing}, 17(4):395--416, 2007.

\end{thebibliography}

\end{document}